\documentclass[twoside]{article}

\usepackage[accepted]{aistats2021_for_arxiv}
\usepackage[utf8]{inputenc} 
\usepackage[T1]{fontenc}    
\usepackage{hyperref}       
\usepackage{url}            
\usepackage{booktabs}       
\usepackage{amsfonts}       
\usepackage{nicefrac}       
\usepackage{microtype}      
\usepackage{mathtools}

\usepackage{amsmath, amsfonts, amssymb, amsthm}
\usepackage{natbib}
\usepackage{algorithm,algorithmic}
\usepackage{fancyhdr}
\usepackage{verbatim}
\usepackage{siunitx}
\usepackage{etoolbox}
\usepackage{bm}

\usepackage{ifthen}
\usepackage{xstring}
\usepackage{calc}
\usepackage{bbm,tikz-cd}

\usepackage{natbib}            \usepackage{setspace}
\usepackage{rotating}
\usepackage{thmtools}
\usepackage{framed}

\usepackage{graphicx}
\usepackage{subcaption}

\def\R{\mathbb{R}}
\def\E{\mathbb{E}}

\def\I{{\mathcal I}}
\def\F{{\mathcal F}}
\def\O{\mathcal O}

\def\bbx{\bar{\mathbf x}}

\def\bbg{\bar{\mathbf g}}

\def\e{\boldsymbol \epsilon}

\def\W{{\bf W}}
\def\X{{\bf X}}
\def\Y{{\bf Y}}
\def\G{{\bf G}}
\def\1{{\bf 1}}

\def\GM{{\mathcal G}}

\def\bx{{\mathbf x}}

\def\bz{{\mathbf z}}
\def\bg{{\mathbf g}}

\def\by{{\mathbf y}}

\def\bI{{\bf I}}

\allowdisplaybreaks

\newtheorem{assumption}{Assumption}

\newtheorem{lemma}{Lemma}
\newtheorem{theorem}{Theorem}
\newtheorem{corollary}{Corollary}
\newtheorem*{skeproof*}{Proof Sketch}

\begin{document}

%

%

\twocolumn[

\aistatstitle{Communication-efficient Decentralized Local SGD over Undirected Networks}

\aistatsauthor{Tiancheng Qin\And S. Rasoul Etesami\And César A. Uribe}

\aistatsaddress{ UIUC \\\texttt{tq6@illinois.edu} \And UIUC \\\texttt{etesami@illinois.edu} \And  MIT \\\texttt{cauribe@mit.edu} } ]

\begin{abstract}
We consider the distributed learning problem where a network of $n$ agents seeks to minimize a global function $F$. Agents have access to $F$ through noisy gradients, and they can locally communicate with their neighbors a network. We study the Decentralized Local SDG method, where agents perform a number of local gradient steps and occasionally exchange information with their neighbors. Previous algorithmic analysis efforts have focused on the specific network topology (star topology) where a leader node aggregates all agents' information. We generalize that setting to an arbitrary network by analyzing the trade-off between the number of communication rounds and the computational effort of each agent. We bound the expected optimality gap in terms of the number of iterates $T$, the number of workers $n$, and the spectral gap of the underlying network. Our main results show that by using only $R=\Omega(n)$ communication rounds, one can achieve an error that scales as $O({1}/{nT})$, where the number of communication rounds is independent of $T$ and only depends on the number of agents. Finally, we provide numerical evidence of our theoretical results through experiments on real and synthetic data.
\end{abstract}

\section{Introduction}\label{sec:intro}

Stochastic Gradient Descent  (SGD) is arguably the most commonly used algorithm for the optimization of parameters of machine learning models.  SGD tries to minimize a function $F$ by iteratively updating parameters as: $\mathbf{x}^{t+1}  = \mathbf{x}^{t} - \eta_t \hat{\mathbf{g}}^{t}$, 
where $\hat{\mathbf{g}}^{t}$ is a stochastic gradient of $F$ at $ \mathbf{x}^{t}$ and $\eta_t$ is the learning rate. However, given the massive scale of many modern ML models and datasets, and taking into account data ownership, privacy, fault tolerance, and scalability, decentralized training approaches have recently emerged as a suitable alternative over centralized ones, e.g., parameter server~\cite{dean2012large}, federated learning~\cite{konevcny2016federated,mcmahan2016federated,kairouz2019advances}, decentralized stochastic gradient descent  \cite{lian2017can,koloskova2019decentralized,assran2018asynchronous}, decentralized momentum SGD \cite{yu2019linear}, decentralized ADAM \cite{nazari2019dadam}, among many others \cite{lian2018asynchronous,koloskova2019decentralized,tang2018communication,lu2020moniqua,tang2019doublesqueeze}.

A naive parallelization of the SGD consists of having multiple workers computing stochastic gradients in parallel, with a central node, or fusion center, where local gradients are aggregated and sent back to the workers. Ideally, such aggregation provides an estimate of the true gradient of $F$ with a lower variance. However, such a structure induces a large communication overhead, where at each iteration of the algorithm, all workers need to send their gradients to the central node, and then the central node needs to send the workers the aggregated information. Local SGD (Parallel SGD or Federated SGD) \cite{stich2018local,zinkevich2010parallelized} presents a suitable solution to reduce such communication overheads. Specifically, each machine independently runs SGD locally and then aggregates by a central node from time to time only. Formally, in Local SGD, each machine runs $K$ steps of SGD locally, after which an aggregation step is made. In total, $R$ aggregation steps are performed. This implies that each machine computes $T=KR$ stochastic gradients and executes $KR$ gradient steps locally, for a total of $N = KRn$ over the set of $n$ machines or workers. We refer interested readers to \cite{wang2018cooperative,koloskova2020unified,lu2020mixml} for a number of recent unifying approaches for distributed SGD. 

The main advantage of the parallelization approach is that it allows distributed stochastic gradient computations. Such an advantage has been recently shown to translate into linear speedups with respect to the nodes available~\cite{spiridonoff2020local, koloskova2019decentralized}. However, linear speedups come at the cost of an increasing number of communication rounds. Recently, in~\cite{khaled2019tighter,stich2019error}, the authors showed that the number of communication rounds required is $\Omega(n \ \text{polylog}(T))$ for strongly convex functions and all agents getting stochastic gradients from the same function. Later in~\cite{spiridonoff2020local}, the authors showed that such communication complexity could be further improved to $\Omega(n)$, still maintaining the linear speedup in the number of nodes $n$.

This paper focuses on smooth and strongly convex functions with a general noise model, where agents have access to stochastic gradients of a function $ F $, but no central node or fusion center exists. Instead, agents can communicate or exchange information with each other via a network. This network imposes communication constraints because a worker or node can only exchange information with those directly connected.

\subsection{Related Work} Our work builds upon two main recent results on the analysis of Local SGD and Decentralized SGD, namely~\cite{koloskova2020unified}, and~\cite{spiridonoff2020local}. In~\cite{koloskova2020unified}, the authors introduced a unifying theory for decentralized SGD and Local updates. In particular, they study the problem of decentralized updates where agents are connected over an arbitrary network. Thus, there is no central or aggregating node, as in leader/follower architectures. Agents communicate with their local neighbors defined over a graph that imposes communication limits between the agents.
Moreover, time-varying networks are also allowed. On the other hand, in~\cite{spiridonoff2020local}, the authors propose a new analysis of Local SGD, where a fixed leader/follower network topology is used, and the central node works as a fusion center. However, \cite{spiridonoff2020local}~introduces a new analysis that shows that linear speedups with respect to the number of workers $n$ can be achieved, with a communication complexity proportional to $\Omega(n)$ only. Our work also relates to Cooperative SGD~\cite{wang2018cooperative}, where smooth non-convex functions are studied. However, no linear improvement is shown, and we provide better communication complexity. A similar problem was also studied in~\cite{li2019communication}, where Local Decentralized SGD was proposed, which allows for multiple local steps and multiple Decentralized SGD steps. Smooth non-convex functions with bounded variance, and heterogeneous agents are also studied.

Our proposed method is not a strict uniform improvement over~\cite{koloskova2020unified,spiridonoff2020local,wang2018cooperative,li2019communication}. Instead, we study a specific regime: arbitrarily fixed network architectures, which is more general than leader/follower topologies. However, fixed networks is a stronger assumption than the time-varying assumption used in~\cite{koloskova2020unified} where arbitrary changes are allowed. Also, we focus on strongly convex and smooth functions, whereas \cite{wang2018cooperative,li2019communication} study non-convex smooth functions, and~\cite{koloskova2019decentralized} studies convex and non-convex functions as well. Moreover, we focus on the homogeneous case, where all agents observe noisy gradients of the same function~\cite{shamir2014distributed,rabbat2015multi}. The more general non-IID data setup was studied in~\cite{li2019convergence,koloskova2020unified,woodworth2020minibatch}. Finally, we show that linear improvement can be achieved with a number of communication rounds that depends only on $n$, which improves the communication complexity of~\cite{koloskova2020unified}.

\subsection{Contributions and Organization}
We confirm the remark of~\cite{koloskova2020unified} that suggested that when all agents obtain noisy gradients from the same function, an improvement over the communication complexity can be made as the lower bound in~\cite{koloskova2020unified} becomes vacuous.  We study the case where the length of the local SGD update is fixed and where this length is increasing. In particular, we show that a linear speedup in the number of agents can be shown for a fixed number of communication rounds proportional to $n$, while only increasing the number of iterations for a general fixed network architecture ~\cite{spiridonoff2020local}. Moreover, we provide numerical experiments that verify our theoretical results on a number of scenarios and graph topologies.

The paper is organized as follows. Section~\ref{sec:problem} describes the problem statement, and assumptions. Section~\ref{sec:results} states our main results and proof sketches. Section~\ref{sec:numerics} provides numerical experiments. Finally, Section~\ref{sec:conclusions} describes concluding remarks and future work.

\textbf{Notation:}
For a positive integer number $n\in \mathbb{Z}_+$, we let $[n]\triangleq\{1, \ldots, n\}$. For an integer $z\in \mathbb{Z}$, we denote the largest integer less than or equal to $z$ by $\lfloor z\rfloor$. We let $\boldsymbol{1}_n$ be an $n$-dimensional column vector with all entries equal to one. A nonnegative matrix $\W=[w_{ij}] \in \R^{n {\times} n}_+$ is called doubly stochastic if it is symmetric and the sum of the entries in each row equals $1$. We denote the eigenvalues of a doubly stochastic matrix $\W$ by $1 = |\lambda_1| \ge |\lambda_2| \ge \cdots \ge |\lambda_n| \ge 0$ and its spectral gap by $1- \rho$, where $\rho\triangleq |\lambda_2|$ denotes the size of the second largest eigenvalue of $\W$. 

\section{Problem Formulation}\label{sec:problem}

Let us consider a decentralized system with a set of $[n]=\{1, \ldots, n\}$ workers. We assume that the workers are connected via a weighted connected undirected graph $\GM = ([n], \W)$, where there is an edge between workers $i$ and $j$ if and only if they can directly communicate with each other. Here, $\W=[w_{ij}]$ is a symmetric doubly stochastic matrix whose $ij$-th entry $w_{ij}> 0$ denotes the weight on the edge $\{i,j\}$, and $w_{ij}=0$ if there is no edge between $i$ and $j$. We note that since $\W$ is the weighted adjacency matrix associated with the \emph{connected} network $\GM$, we have $\rho \in [0, 1)$, where $1-\rho$ is the spectral gap. The specific influence of the graph topology on the graph can be found in the literature~\cite{nedic2019graph}. For example, for path graphs one usually has $\rho = O(1/n^2)$, whereas for well connected graphs, like Erd\H{o}s-Reny\'i random graph, this influence reduces to $\rho = O(1/\log^2(n))$. 

The workers' objective is to minimize a global function $F:\R^d \rightarrow \R$ by performing local gradient steps and occasionally communicating over the network and leveraging the samples obtained by the other agents. Similarly, as in~\cite{spiridonoff2020local}, we assume that the function $F$ is a smooth function and that all workers have access to $F$ through noisy gradients. More precisely, we consider the following assumptions throughout the paper.

Assumption~\ref{as: F} specifies the class of functions we will be working with (i.e., strongly convex and smooth), and Assumption~~\ref{as: g} describe our noise model.

\begin{assumption}\label{as: F}
	Function $F:\R^d\rightarrow \R$ is differentiable, $\mu$-strongly convex and $L$-smooth with condition number $\kappa \triangleq {L}/{\mu}\ge 1$, that is, for every $x,y\in \R^d$ we have,
	\begin{align*}
		\frac{\mu}{2}\Vert \bx {-} \by \Vert^2 \leq F(\by) {-} F(\bx) {-} \langle \nabla F(\bx), \by {-} \bx \rangle \leq \frac{L}{2}\Vert \bx {-} \by \Vert^2.
	\end{align*}
\end{assumption}

\begin{assumption}\label{as: g}
	Each worker $i$ has access to a gradient oracle which returns an unbiased estimate of the true gradient in the form $\bg_i(\bx) = \nabla F(\bx) + \e_i$, such that $\e_i$ is a zero-mean conditionally independent random noise with its expected squared norm error bounded as
	\begin{align*}
		\E[\e_i] = \mathbf 0, \qquad \E[\Vert \e_i \Vert^2|\bx] \leq c\Vert \nabla F(\bx) \Vert^2 + \sigma^2,
	\end{align*}
	where $\sigma^2,c\geq0$ are constants.
\end{assumption}

Let us denote the length of time horizon by $T\in\mathbb{Z}_+$. In Decentralized Local SGD, each worker $i\in [n]$ holds a local parameter $\bx_i^t$ at iteration $t$ and a set $\I \subset [T]$ of communication times. By writing all the variables and the gradient values at time $t$ in a matrix form, we have 
\begin{align}\nonumber
&\X^t  \triangleq \left[ \bx_{1}^t, \cdots, \bx_{n}^t \right] \in \R^{d {\times} n},\cr 
&\G(\X^t)  \triangleq \left[ \bg_{1}^t, \cdots, \bg_{n}^t \right] \in \R^{d {\times} n},
\end{align}
where $\bg_i^t \triangleq \bg(\bx_i^t)$ is the stochastic gradient of node $i$ at iteration $t$. In each time step $t$, the Decentralized Local SGD algorithm performs the following update:
\[
\X^{t+1} = (\X^{t} - \eta_t\G(\X^t))\W_t,
\]
where  $\W_t \in \R^{n \times n}$ is the connected matrix defined by
\begin{equation}
	\label{eq:W_main}
	\W_t = \left\{ \begin{array}{ll}
		\bf{I}_n & \text{if} \ t  \notin \I,  \\
		\W & \text{if} \ t  \in \I.
	\end{array}\right. 
\end{equation}
Note that according to \eqref{eq:W_main}, the workers simply update their variables based on SGD at each $t\notin \I$, and communicate with the neighbors only at time instances $t\in \I$. The pseudo code for the Decentralized Local SGD is provided in Algorithm \ref{alg1}.
\begin{algorithm}[t]
	\caption{Decentralized Local SGD}
	\begin{algorithmic}[1]\label{alg1}
		\STATE Input: $\bx_i^0 = \bx^0$ for $i \in [n]$, total number of iterations $T$, the step-size sequence $\{\eta_t\}_{t=0}^{T-1}$ and $\I \subseteq [T]$.
		\FOR{$t=0,\ldots,T-1$}
		\FOR{$j=1,\ldots,n$}
		\STATE evaluate a stochastic gradient $\bg_j^t$
		\ENDFOR
		\FOR{$j=1,\ldots,n$}
		\IF{$t+1 \in \mathcal{I}$}
		\STATE $\bx_j^{t+1} = \sum_{i=1}^n w_{ij}(\bx_i^t - \eta_t \bg_i^t)$
		\ELSE
		\STATE $\bx_j^{t+1} = \bx_j^t - \eta_t \bg_j^t$
		\ENDIF
		\ENDFOR
		\ENDFOR
	\end{algorithmic}
\end{algorithm}

Our main objective will be to analyze the communication complexity and sample complexity of the outputs of Algorithm~\ref{alg1}. In the next section, we state our main results.

\section{Main Results: Convergence and Communication Complexity}\label{sec:results}

In this section, we analyze the convergence guarantee of the Decentralized Local SGD in terms of the number of communication rounds. We will show that for a specific choice of inter-communication intervals, one can achieve an approximate solution with arbitrarily small optimality gap with a convergence rate of $O(\frac{1}{nT})$ using only $|\I|=\Omega(n)$ communications rounds, which is independent of the horizon length $T$. To that end, let us consider the following notations that will be used throughout the paper.
\begin{align*}
    &\bbx^t\!\triangleq  \frac{1}{n}\X^t\1_n, \ \ \ \ \ \ \ \ \ \bar \e^t \triangleq \frac{1}{n}\sum_{i=1}^n\e_i^t, \cr &\bbg^t \triangleq \frac{1}{n}\G(\X^t)\1_n, \ \ \  \overline{\nabla F}(\X) \triangleq  \frac{1}{n}\nabla F(\X)\1_n,
\end{align*}
and define $\xi^t\triangleq \E[F(\bbx^t)] - F^*$ to be the optimality gap of the solution at time $t$. Let $0<\tau_1 < \ldots < \tau_R \le  T$ be the communication times, and denote the length of $(i+1)$-th inter-communication interval by $H_i \triangleq \tau_{i+1} - \tau_i$, for $i=0,\ldots,k-1$. Moreover, we let $\rho_t$ be the size of the second largest eigenvalue of the connectivity matrix $\W_t$, i.e.,
\begin{equation*}
	\rho_t = \left\{ \begin{array}{ll}
		1 & \text{if} \ t  \notin \I,  \\
		\rho & \text{if} \ t  \in \I.
	\end{array}\right. 
\end{equation*}
Note that although we have decided to define $\W_t$ as a time-varying matrix, this is done for notation convenience only. Such representation allows to model the scenario we are interested in by assuming that in all time instances where no communication over the network is made, the effective adjacency matrix of the communication network is the identity matrix.

We can now state the following theorem that bounds the optimality error in terms of the communication network's spectral gap at different times.

\begin{theorem}\label{thm1}
	Let Assumptions \ref{as: F} and \ref{as: g} hold, choose  $\beta \geq \max \{9\kappa^2 c \ln(1 + T/2\kappa^2) + 2\kappa(1+c/n), 2\kappa^2\}$, and set $\eta_k = 2/ (\mu (k+\beta))$. Then, the output of Algorithm~\ref{alg1} has the following property:
	\begin{align}\label{eq: opt E3}
		 & \E[F(\bbx^T)] - F^*  \leq \frac{\beta^2 (F(\bbx^0) - F^*)}{T^2}  +\nonumber \\
		& \qquad \qquad +  \frac{2L \sigma^2}{n \mu^2 T} + \frac{9L^2 \sigma^2}{ \mu^3 T^2} \sum_{t=0}^{T-1} \frac{1}{t+\beta}\sum_{k = 0}^{t-1}\prod_{i=k}^{t-1}\rho_i^2,
	\end{align}
	where $F^* \triangleq \min_{\bx} F(\bx)$ is the minimal value of the objective function.
\end{theorem}

Theorem~\ref{thm1} states a general convergence result for an arbitrary sequence of communication graphs. Note that the expected optimality gap between the function value is evaluated at the average iterate and the minimal function value for any sequence of spectral gaps. We believe that this result is of independent interest, as it allows further flexibility in the analysis. Later in Section~\ref{sec:results}, we will specialize in this result to two main cases: having fixed communication intervals and having increased-sized communication intervals.

To prove Theorem~\ref{thm1}, we state two main technical lemmas, which will be proven in the supplementary material. Let us first define the following parameters:
\begin{align}\nonumber
G_t \triangleq \E[\frac{1}{n}\sum_{i=1}^n\|\nabla F(\bx^t_i)\|_2^2],\  V_t \triangleq\E[\frac{1}{n}\sum_{i=1}^n\|\bx^t_i-\bbx^t\|_2^2].
\end{align}
It is shown in the supplementary material that:

\begin{lemma}\label{lem1}
	Let Assumptions \ref{as: F} and \ref{as: g} hold. Then,
	\begin{align*}
		\xi^{t+1} &\leq \xi^t(1 - \mu \eta_t) - \frac{\eta_t}{2} G_t+ \frac{\eta_t^2 L}{2} (1+\frac{c}{n})G_t\cr 
		&\qquad+\frac{\eta_tL^2}{2}V_t +\frac{\eta_t^2L\sigma^2}{2n}.
	\end{align*}
\end{lemma}

\begin{lemma}\label{lem4}
	Let assumptions of Theorem \ref{thm1} hold. Then,
	\begin{align*}
		V_t \leq
		\frac{9(n-1)}{n}\sum_{k=0}^{t-1}\frac{cG_k+\sigma^2}{\mu^2(t+\beta)^2}\prod_{i=k}^{t-1}\rho_i^2.
	\end{align*}
\end{lemma}

We are now ready to prove our main result.

\begin{skeproof*}[Theorem~\ref{thm1}]
Using Lemmas \ref{lem1} and \ref{lem4}, one can bound the optimality error using $G_t$ and $V_t$ as 	
\begin{align}\nonumber
\xi^{t+1} &\leq \xi^t(1 - \mu \eta_t) - \frac{\eta_t}{2} G_t+ \frac{\eta_t^2 L}{2} (1+\frac{c}{n})G_t\cr 
		&\qquad+\frac{\eta_t^2L\sigma^2}{2n}+\frac{\eta_tL^2}{2}V_t, \cr 
	   V_t &\leq
		\frac{9(n-1)}{n}\sum_{k=0}^{t-1}\frac{cG_k+\sigma^2}{\mu^2(t+\beta)^2}\prod_{i=k}^{t-1}\rho_i^2.
	   	\end{align}

	If we combine the above relations and substitute \mbox{$\eta_t = {2}/{\mu(t+\beta)}$}, we get
	\begin{align*}
		\xi^{t+1} &\leq \xi^t(1 - \mu \eta_t)\! -\! \frac{1}{\mu(t+\beta)} G_t \!+\! \frac{2L}{\mu^2(t+\beta)^2} (1+\frac{c}{n}) G_t\cr  &+\frac{2L\sigma^2}{n\mu^2(t+\beta)^2}+ 
		\frac{9L^2}{\mu^3(t+\beta)^3} 
		\sum_{k=0}^{t-1} (cG_k+\sigma^2)\prod_{i=k}^{t-1}\rho_i^2.
	\end{align*}
	
	Moreover, if we multiply both sides of the above inequality by $(t+\beta)^2$ and use the following valid inequality
	\begin{align*}
	(1-\mu \eta_t)(t+\beta)^2 &= (1 - \frac{2}{t+\beta})(t+\beta)^2 \cr 
	&= (t+\beta)^2 - 2(t+\beta) \cr 
	&< (t+\beta-1)^2,
	\end{align*}
we obtain,
	\begin{align*}
		\xi^{t+1} (t\!+\!\beta)^2 \!\leq\! \xi^t (t\!+\!\beta\!-\!1)^2 \!+\! \left( \frac{2L}{\mu^2}(1+\frac{c}{n}) \!-\! \frac{t+\beta}{\mu} \right)G_t\cr
		+ \frac{2L \sigma^2}{n \mu^2} + 
		\frac{9L^2}{\mu^3(t+\beta)} \sum_{k=0}^{t-1} (c\E[G^k]+\sigma^2)\prod_{i=k}^{t-1} \rho_i^2.
	\end{align*}
	Summing this relation for $t=0,\ldots,T-1$, we get
	\begin{align*}
		&\xi^{T}(T + \beta -1 )^2 \leq \xi^{0}(\beta -1 )^2  +  \frac{2L \sigma^2}{n \mu^2}T\cr 
		&+ \frac{9L^2\sigma^2}{\mu^3}\sum_{t=0}^{T-1} \frac{1}{t+\beta}\sum_{k = 0}^{t-1}\prod_{i=k}^{t-1}\rho_i^2 \cr
		&+\! \sum_{t=0}^{T-1}G_t\! \left( \sum_{k=t+1}^{T-1}\!\frac{9L^2 c}{\mu^3(k+\beta)}\prod_{i=t}^{k-1}\rho_i^2\! +\! \frac{2L}{\mu^2}(1+\frac{c}{n}) \!-\! \frac{t+\beta}{\mu} \right).
	\end{align*}
	We can now bound the coefficient of $G_t$ in the above expression as follows:
	\begin{align*}
		&\sum_{k=t+1}^{T-1}\frac{9L^2 c}{\mu^3(k+\beta)}\prod_{i=t}^{k-1}\rho_i^2 + \frac{2L}{\mu^2}(1+\frac{c}{n}) - \frac{t+\beta}{\mu} \cr 
		&\leq 
		\sum_{k=t+1}^{T-1}\frac{9L^2 c}{\mu^3(k+\beta)} + \frac{2L}{\mu^2}(1+\frac{c}{n}) - \frac{t+\beta}{\mu} \\
		&\le \sum_{k=1}^{T-1}\frac{9L^2 c}{\mu^3(k+\beta)} + \frac{2L}{\mu^2}(1+\frac{c}{n}) - \frac{\beta}{\mu}\\
		&\le \frac{9L^2 c}{\mu^3}\ln(\frac{T-1+\beta}{\beta})+ \frac{2L}{\mu^2}(1+\frac{c}{n}) - \frac{\beta}{\mu}\\
		& = \frac{1}{\mu}\left( 9\kappa^2c \ln(1 + \frac{T- 1}{\beta}) + 2\kappa(1+\frac{c}{n}) - \beta\right)\leq 0,
	\end{align*}
	where in the third inequality we have used $\sum_{k=t_1+1}^{t_2} {1}/{k} \leq \int_{t_1}^{t_2} {dx}/{x} = \ln({t_2}/{t_1})$. The last inequality also holds by the assumption of the theorem. As the coefficient of $G_t$ is non-positive, we can simply drop it from the upper bound to get,
	\begin{align*}
		\xi^{T}(T + \beta -1 )^2 &\leq \xi^{0}(\beta -1 )^2  +  \frac{2L \sigma^2}{n \mu^2}T\cr 
		&
		\qquad+ \frac{9L^2\sigma^2}{\mu^3}\sum_{t=0}^{T-1} \frac{1}{t+\beta}\sum_{k = 0}^{t-1}\prod_{i=k}^{t-1}\rho_i^2.
	\end{align*}
	Finally, dividing both sides of the above inequality by $(T+\beta -1 )^2$ completes the proof. 
\end{skeproof*}

Next, we specialize Theorem \ref{thm1} to two specific choices of inter-communication time intervals.   
\subsection{Fixed-Length Intervals}
A simple way to select the communication times $\I$, is to partition the entire training time $T$ into $R$ subintervals of length at most $H$, i.e. $\tau_i = iH $ for $i=1,\ldots,R-1$ and $\tau_R=\min\{RH,T\}$. In that case, we can bound the error term in \eqref{eq: opt E3} as follows:
\begin{align*}
	\sum_{k = 0}^{t-1}\prod_{i=k}^{t-1}\rho_i^2 &\le \sum_{l=0}^{\lfloor\frac{t-1}{H}\rfloor}\sum_{k=lH}^{(l+1)H-1}\prod_{i=k}^{t-1}\rho_i^2\\
	&\le H\sum_{l=0}^{\lfloor\frac{t-1}{H}\rfloor}\prod_{i=(l+1)H-1}^{t-1}\rho_i^2\\
	&= H\sum_{l=0}^{\lfloor\frac{t-1}{H}\rfloor}\rho^{2(\lfloor\frac{t-1}{H}\rfloor-l)}\\
	&=H\sum_{l=0}^{\lfloor\frac{t-1}{H}\rfloor}\rho^{2l}\le \frac{H}{1-\rho^2},
\end{align*}
where in the first inequality and by some abuse of notation we set $\prod_{i=k}^{t-1}\rho_i^2=1$ if $k>t-1$. As a result, we can upper-bound the error term in \eqref{eq: opt E3} by 
\begin{align*}
	\sum_{t=0}^{T-1} \frac{1}{t+\beta}\sum_{k = 0}^{t-1}\prod_{i=k}^{t-1}\rho_i^2
	&\leq \frac{H}{1-\rho^2}\sum_{t=0}^{T-1} \frac{1}{t+\beta} \cr 
	&\leq \frac{H}{1-\rho^2} \ln(1+\frac{T}{ \beta - 1}).
\end{align*}
Therefore, we obtain the following corollary.
\begin{corollary}\label{cor:fixed-interval}
	Suppose that the assumptions of Theorem \ref{thm1} hold, and moreover, workers communicate at least once every $H$ iterations. Then,
	\begin{align}\label{eq: fixed-int}
	  \E[F(\bbx^T)] - F^* &\leq \frac{\beta^2(F(\bbx^0) - F^*)}{T^2}+\frac{2L \sigma^2}{n \mu^2 T} \cr
		& + \frac{9L^2 \sigma^2 H}{ \mu^3 T^2(1-\rho^2)} \ln(1 + \frac{T}{\beta - 1}).
	\end{align}
\end{corollary}

Using Corollary \ref{cor:fixed-interval}, if we choose $H= \O(\frac{T}{n \ln(T)})$, then Algorithm 1 achieves a linear speedup in the number of workers, which is equivalent to a communication complexity of $R = \Omega(n \ln(T))$.

\subsection{Varying Intervals}
Here, we consider a more interesting choice of communication times with varying interval length. In other words, we allow the length of consecutive inter-communication intervals $H_i \triangleq \tau_{i+1} - \tau_i$ to grow linearly over time. The following Theorem presents a performance guarantee for this choice of communication times.

\begin{theorem}\label{thm2}
	Suppose assumptions of Theorem \ref{thm1} hold. Choose the maximum number of communications $1 \leq R \leq \sqrt{2T}$ and set $a\triangleq\lceil 2T/R^2 \rceil \geq 1$, $H_i = a(i+1)$ and $\tau_{i} = \min(a {i(i+1)}/{2}, T)$ for $i=1,\ldots,R$. Then, using Algorithm \ref{alg1}, we have
	\begin{align}\label{eq: opt general linear H}
		 \E[F(\bbx^T)] - F^* &\leq \frac{ \beta^2 (F(\bbx^0) - F^*)}{T^2}+ \frac{2L \sigma^2}{n \mu^2 T} \cr 
		&\qquad+ \frac{144L^2 \sigma^2}{(1-\rho^2)\mu^3 TR}.
	\end{align}
\end{theorem}
\begin{proof}
	Define $\tau_0 = 0$, for any $t$ satisfying $\tau_j\le t<\tau_{j+1}$. We can write, 
	\begin{align*}
		\sum_{k = 0}^{t-1}\prod_{i=k}^{t-1}\rho_i^2 &\le \sum_{l=0}^{j}\sum_{k=\tau_l}^{\tau_{l+1}-1}\prod_{i=k}^{t-1}\rho_i^2\\
		&\le H_j\sum_{l=0}^{j}\prod_{i=\tau_{l+1}-1}^{t-1}\rho_i^2\\
		&= H_j\sum_{l=0}^{j}\rho^{2(j-l)}\\
		&\le \frac{H_j}{1-\rho^2}.
	\end{align*}
	Therefore,
	\begin{align*}
		\sum_{t=0}^{T-1} \frac{1}{t+\beta}\sum_{k = 0}^{t-1}\prod_{i=k}^{t-1}\rho_i^2 &\le \sum_{t=\tau_0}^{\tau_1-1}\frac{1}{t+\beta}\sum_{k = 0}^{t-1}\prod_{i=k}^{t-1}\rho_i^2\cr 
		&\qquad+ \sum_{j=1}^{R-1} \sum_{t=\tau_j}^{\tau_{j+1}-1} \frac{1}{t+\beta} ( \sum_{k = 0}^{t-1}\prod_{i=k}^{t-1}\rho_i^2) \\
		&\hspace{-0.7cm}\leq \sum_{t=\tau_0}^{\tau_1-1}\frac{t}{t+\beta}+\frac{1}{1-\rho^2} \sum_{j=1}^{R-1} \sum_{t=\tau_j}^{\tau_{j+1}-1} \frac{H_j}{t+\beta} \\
		&\hspace{-0.7cm}\le H_0+\frac{1}{1-\rho^2} \sum_{j=1}^{R-1}\frac{H_j^2}{\tau_j+\beta} \\
		&\hspace{-0.7cm}= a + \frac{1}{1-\rho^2}\sum_{j=1}^{R-1} \frac{2a^2(j+1)^2}{a j(j+1) + 2 \beta} \\
		&\hspace{-0.7cm}\leq a + \frac{2}{1-\rho^2}\sum_{j=1}^{R-1} \frac{a^2 (j+1)^2 }{aj(j+1)} \\
		&\hspace{-0.7cm}\leq \frac{4aR}{1-\rho^2}.
	\end{align*}
	If we substitute the values of $R$ and $a$ into the above relation, we get
	\begin{align*}
\sum_{t=0}^{T-1} \frac{t-\tau(t)}{t+\beta} &\leq \frac{4aR}{1-\rho^2} \leq \frac{4(\frac{2T}{R^2}+1)R}{1-\rho^2}\cr 
	&=\frac{1}{1-\rho^2} (\frac{8T}{R} + 4R)\cr 
	&\leq  \frac{16T}{(1-\rho^2)R},
	\end{align*} 
	where the last inequality holds because $R\leq \sqrt{2T}$. The above relation, together with Theorem \ref{thm1}, concludes the proof.
\end{proof}

According to Theorem \ref{thm2}, if we choose the number of communication rounds to be $R = \Omega(n)$, then Algorithm 1 achieves an error that scales as $O(\frac{1}{nT})$ in the number of workers when $T=\Omega(n^2)$. As a result, we obtain a linear speedup in the number of workers by simply increasing the number of iterations while keeping the total number of communications bounded.

\section{Numerical Experiments}\label{sec:numerics}

This section shows the results of several numerical experiments to compare the performance of different communication strategies in Decentralized Local SGD and show the impact of the number of workers and the communication network structure.

 \subsection{Quadratic Function With Strong-Growth Condition}
We use the same cost function $F$ as in \cite{spiridonoff2020local} for performance evaluation: define $F(\bx) = \E_\zeta f(\bx, \zeta)$ where,
\begin{align}\label{eq: f-zeta}
    f(\bx, \zeta) \triangleq \sum_{i=1}^d \frac{1}{2}x_i^2 (1 + z_{1,i}) + \bx^\top \bz_2.
\end{align}
Here, $\zeta = (\bz_1, \bz_2)$, where $\bz_1, \bz_2 \in \R^d, z_{1,i}\sim \mathcal{N}(0,c_1)$ and $z_{2,i} \sim \mathcal{N}(0,c_2)$, $\forall i \in[d]$ are random variables with normal distributions. We assume at each iteration $t$, each worker $i$ samples a $\zeta_i^t$ and uses $\nabla f(\bx, \zeta_i^t)$ as a stochastic estimate of $\nabla F(\bx)$.  It is easy to verify that $F(\bx) = \frac{1}{2}\bx^2$ is $1$-strongly convex and $1$-smooth, and $F^* = 0$. Moreover, the noise variance is uniform with strong-growth condition (Assumption \ref{as: g}): $\E_\zeta[\Vert \nabla f(\bx, \zeta) - \nabla F(\bx) \Vert^2] = c \Vert \nabla F(\bx) \Vert^2 + \sigma^2$, where $c = c_1$ and $\sigma^2 = dc_2$.

To compare different communication strategies in Decentralized Local SGD, we set the number of workers to be $n=20$ and generate connected random communication graphs using an Erd$\ddot{\text{o}}$s-R$\grave{\text{e}}$nyi graph with the probability of connectivity $p = 0.3$. We use the  local-degree weights (also known as Metropolis weights) to generate the mixing matrix, i.e., assigning the weight on an edge based on the larger degree of its two incident nodes \cite{xiao2004fast}:
$$w_{ij} = \frac{1}{\max\{d_i,d_j\}},\  \text{ if $\{i,j\}$ is an edge.}$$
We use Decentralized Local SGD (Algorithm~\ref{alg1}) to minimize $F(\bx)$ using different communication strategies. We select $c_1=15, c_2=\frac{1}{12}, d=20$, and $T=2000$ iterations, and the step-size sequence $\eta_t = 2/\mu(t+\beta)$ with $\beta=1$. We start each simulation from the initial point of $\bx^0 = \mathbf 1_d$ and repeat each simulation $500$ times. The average of the results are reported in Figures \ref{fig: res2}(a) and \ref{fig: res2}(b). 

\begin{figure}[t!]
	\centering
	\begin{subfigure}[b]{\linewidth}
		\includegraphics[width=\linewidth]{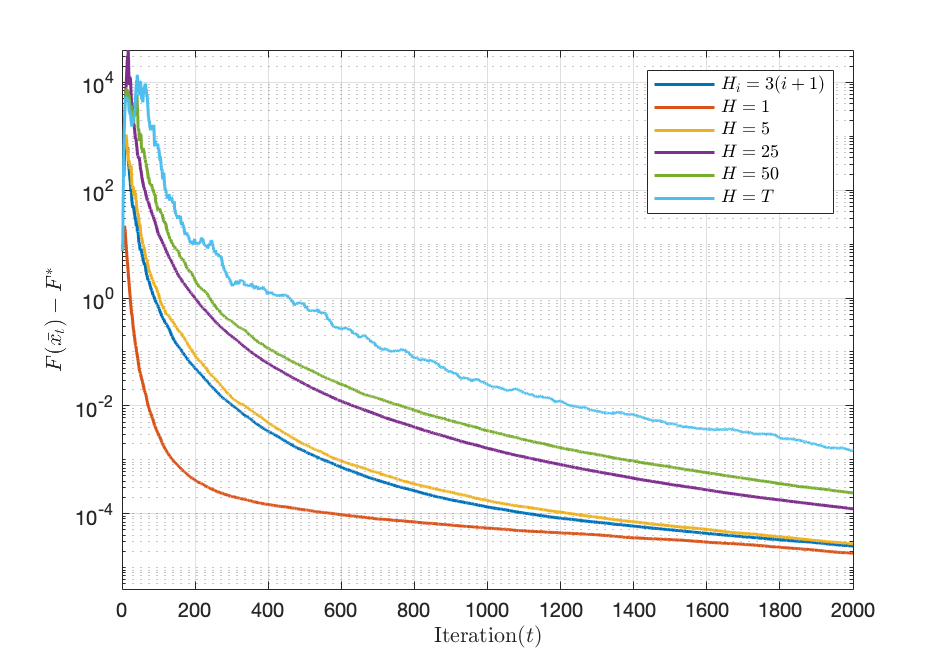}\vspace{-0.3cm}
        \caption{Error over iterations.}
	\end{subfigure}
	\begin{subfigure}[b]{\linewidth}
		\includegraphics[width=\linewidth]{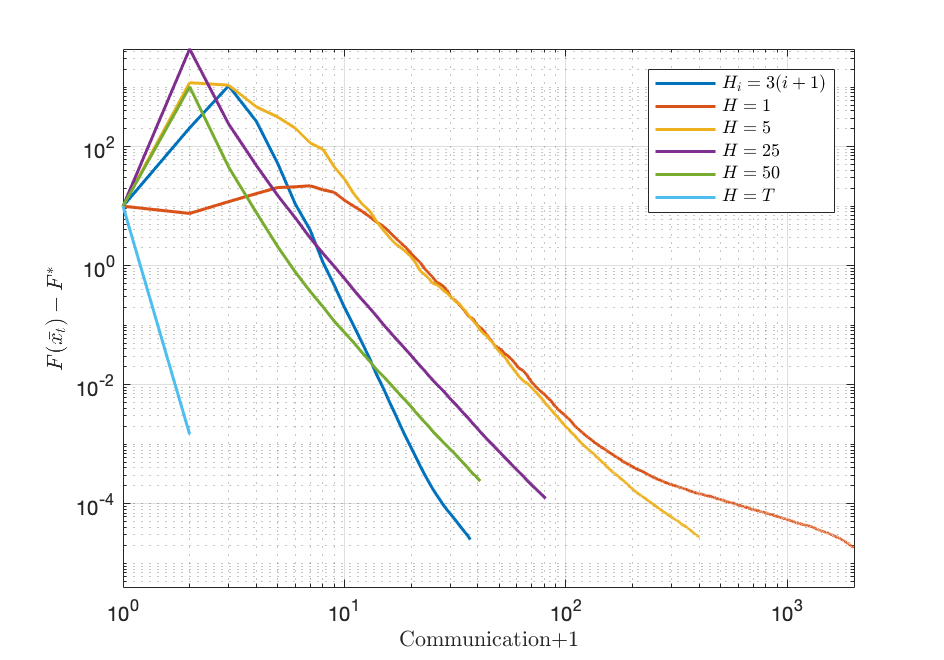}\vspace{-0.3cm}
        \caption{Error over communications}
	\end{subfigure}
	\caption{Local SGD with different communication strategies with $F(\bx)=\E_\zeta f(\bx, \zeta)$ defined in \eqref{eq: f-zeta}, $c_1=15, c_2=\frac{1}{12}, d=20, \beta=1$. Figures (a) and (b) show the error of different communication methods over iteration and communication round, respectively, with a fixed network size of $n=20$.}
	\label{fig: res2}
\end{figure}

Figure \ref{fig: res2}(a) shows that considering error over iterations, our proposed strategy, which gradually increases communication intervals ($H_i = 3(i+1)$) outperforms all the other strategies except the one that communicates at every iteration, which requires far more communication rounds than our proposed strategy.

Figure \ref{fig: res2}(b) illustrates the effectiveness of each communication round in different strategies. It shows that our proposed strategy uses communication rounds more efficiently than all the other strategies except that only communicates at the end of optimization, which is not competitive in terms of its final error. 

In particular, it is shown that the strategy with the same number of communications as our proposed strategy but fixed communication intervals ($H=50$) has both higher transient error and final error. This justifies the advantages of having more frequent communication at the beginning of the optimization and gradually increases communication intervals.

\begin{figure}[t!]
	\centering
	\begin{subfigure}[b]{\linewidth}
		\includegraphics[width=\linewidth]{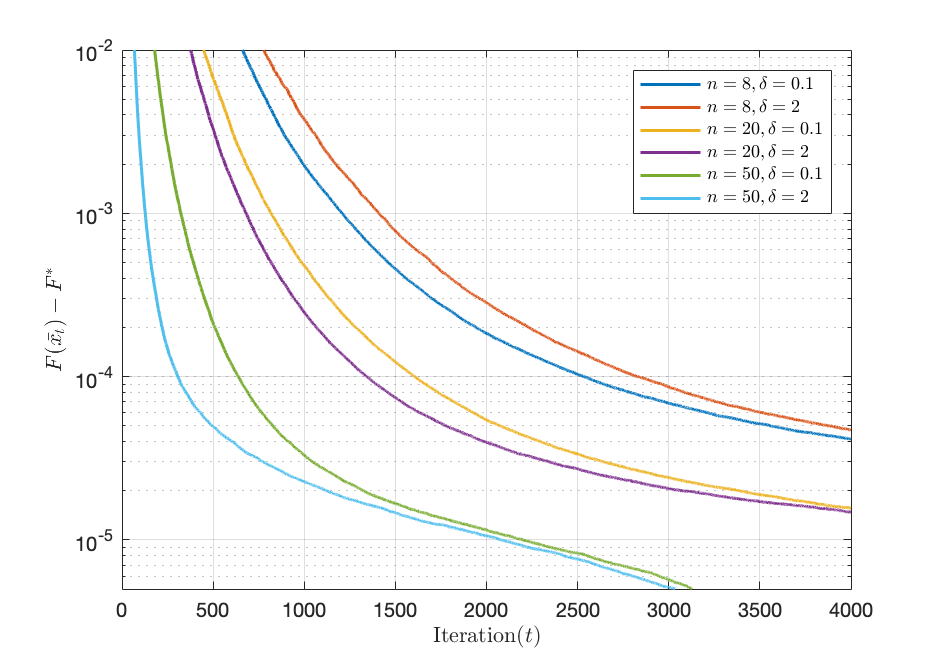}\vspace{-0.3cm}
        \caption{Erd$\ddot{\text{o}}$s-R$\grave{\text{e}}$nyi graphs}
	\end{subfigure}
	\begin{subfigure}[b]{\linewidth}
		\includegraphics[width=\linewidth]{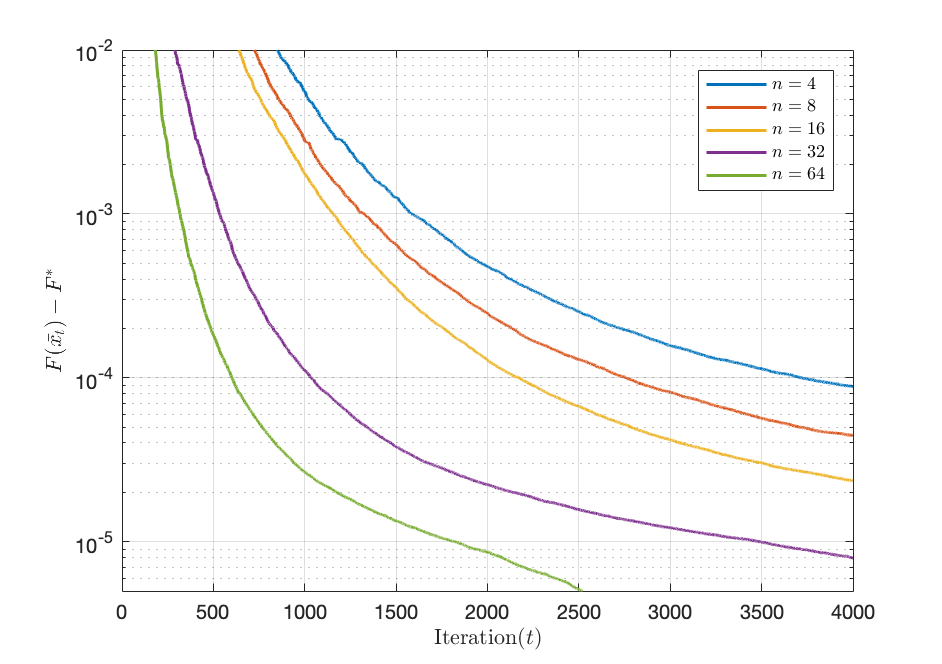}\vspace{-0.3cm}
        \caption{Path graphs}
	\end{subfigure}
	\caption{The convergence of Local SGD with the communication method proposed in this paper ($2n$ communication rounds) for various Erd$\ddot{\text{o}}$s-R$\grave{\text{e}}$nyi graphs with different number of workers and the probability of connectivity (a) and various path graphs with different number of workers (b).}
	\label{fig: res3}
\end{figure}

To evaluate the impact of the number of workers and the communication network structure in Decentralized Local SGD, we generate two sets of communication graphs: various Erd$\ddot{\text{o}}$s-R$\grave{\text{e}}$nyi graphs with different number of workers and the probability of connectivity $p = \bar{d}/n$, where $\bar{d}$ is the average degree of nodes and $\bar{d} = (1+\delta)\ln n$ with $\delta=0.1$, indicating a sparse Erd$\ddot{\text{o}}$s-R$\grave{\text{e}}$nyi graph, or $\delta=2$, indicating a dense Erd$\ddot{\text{o}}$s-R$\grave{\text{e}}$nyi graph; various path graphs with different number of workers. We use the  local-degree weights to generate the mixing matrix.

We use Decentralized Local SGD (Algorithm~\ref{alg1}) to minimize $F(\bx)$ and set the communication strategy to be varying intervals with the number of communication rounds $R=2n$. We select $c_1=15, c_2=\frac{1}{12}, d=20$, and $T=4000$ iterations, and the step-size sequence $\eta_t = 2/\mu(t+\beta)$ with $\beta=1$. We start each simulation from the initial point of $\bx^0 = \mathbf 1_d$ and repeat each simulation $500$ times. The average of the results are reported in Figures \ref{fig: res3}(a) and \ref{fig: res3}(b). 

Figure \ref{fig: res3}(a) shows that while the connectivity of the communication network makes a minor contribution to the differences of the convergence speed of the algorithm, with the better network connectivity generally resulting in faster convergence speed, the significant differences are caused by the different number of workers in the network. A linear-speed up in the number of workers can be seen in the figure with only $R = 2n$ communication rounds.

Figure \ref{fig: res3}(b) further verifies our theoretical findings. Notice that the connectivity of the communication network is reflected by the parameter ${1}/{(1-\rho^2)}$ in \eqref{eq: opt general linear H}, and ${1}/{(1-\rho^2)}$ for each path graph with $n=4,8,16,32,64$ workers can be computed as $2.84,10.1,39.3,156,623$, respectively. However, despite the increase of ${1}/{(1-\rho^2)}$, a linear-speed up of the convergence speed in the number of workers can still be observed. Figures \ref{fig: res3}(a)(b) verify that linear-speed up in the number of workers can be achieved with only $R = 2n$ communication rounds.

Additional experimental results for the regularized logistic regression can be found in the supplementary material.

\section{Conclusion}\label{sec:conclusions}

In this paper, we considered the problem of computation versus communication trade-off for Decentralized Local SGD over arbitrary undirected connected graphs. We have shown that by using appropriately chosen inter-communication intervals, one can achieve a linear speedup in the number of workers while keeping the total number of communications bounded by the number of workers.

In this work, we restricted our attention to undirected networks and homogeneous objective functions. Therefore, generalizing our work to directed networks in which workers have access to heterogeneous objective functions (or heterogeneous data sets) would be an interesting research direction. In that regard, the existing results such as \cite{li2019communication,khaled2019tighter,assran2019stochastic,pu2020push,zhang2019asynchronous} could serve as a good starting point.

 \bibliographystyle{authordate1}
 \bibliography{references,references1}
\newpage

	\onecolumn

	
	\section{Auxiliary results and proofs}
	
	
	Let us define the following notations used in the proofs here:
	\begin{align*}
	G_t := \E[\frac{1}{n}\sum_{i=1}^n\|\nabla F(\bx^t_i)\|_2^2],\qquad V_t :=\E[\frac{1}{n}\sum_{i=1}^n\|\bx^t_i-\bbx^t\|_2^2],
	\qquad \bar \e^t := \frac{1}{n}\sum_{i=1}^n\e_i^t
	\end{align*}
	Moreover, define $\F^t:=\{\bx_i^k, \bg_i^k | 1\leq i \leq n, 0\leq k \leq t-1 \} \cup \{\bx_i^t | 1 \leq i \leq n\}$ to be the history of all the iterates up to time $t$. Then, we can bound the optimality error in terms of $V_t$ and $G_t$ as follows:

	\begin{lemma}\label{lem1}
		Let Assumptions \ref{as: F} and \ref{as: g} hold. Then,
		\begin{align*}
		\xi^{t+1} &\leq \xi^t(1 - \mu \eta_t) - \frac{\eta_t}{2} G_t+ \frac{\eta_t^2 L}{2} (1+\frac{c}{n})G_t+\frac{\eta_tL^2}{2}V_t +\frac{\eta_t^2L\sigma^2}{2n}. 
		\end{align*}
	\end{lemma}
	
	\begin{proof}
		We can write
		\begin{align*}
		\bbx^{t+1} = \frac{1}{n}\X^{t+1}\1_n 
		&= \frac{1}{n}((\X^{t} - \eta_t\G(\X^t))\W_t)\1_n \\
		&=  \frac{1}{n}\X^{t}\W_t\1_n - \eta_t\frac{1}{n}\G(\X^t))\W_t\1_n \\
		&= \bbx^{t}-\eta_t\bbg^t.
		\end{align*}

		By Assumption \ref{as: F}, we have
		\begin{align}\label{eq2}
		\E[F(\bbx^{t+1})-F(\bbx^t)] \leq -\eta_t \E[\langle \nabla F(\bbx^t), \bbg^t \rangle ]+ \frac{\eta_t^2 L}{2} \E[\Vert \bbg^t \Vert_2^2 ].
		\end{align}
		We bound the first term on the right side of \eqref{eq2} by conditioning on $\F^t$ as follows:
		\begin{align} \label{eq3}
		\E[\langle  \nabla F(\bbx^t), \bbg^t \rangle | \mathcal{F}^{t} ] &= \frac{1}{n} \sum_{i=1}^n \langle \nabla F(\bbx^t), \E[\bg_i^t | \bx_i^t] \rangle \nonumber \\
		&= \frac{1}{2} \Vert \nabla F(\bbx^t) \Vert^2 + \frac{1}{2n} \sum_{i=1}^n \Vert \nabla F(\bx_i^t) \Vert^2- \frac{1}{2n} \sum_{i=1}^n \Vert \nabla F(\bbx^t) - \nabla F(\bx_i^t) \Vert^2 \cr  
		& \geq \mu (F(\bbx^t) - F^*) + \frac{1}{2n} \sum_{i=1}^n \Vert \nabla F(\bx_i^t) \Vert^2 - \frac{L^2}{2n} \sum_{i=1}^n \Vert \bbx^t - \bx_i^t \Vert^2,
		\end{align}
		where we used $\langle a, b \rangle  = \frac{1}{2} \Vert a \Vert ^2 + \frac{1}{2} \Vert b\Vert^2 - \frac{1}{2} \Vert a-b \Vert ^2$ in the second equality, and $\frac{1}{2}\Vert \nabla F(\bx) \Vert^2 \geq \mu (F(\bx) - F^*)$ together with smoothness of $F$ in the last inequality.
		Taking expectation from \eqref{eq3}, we obtain
		\begin{align*}
		\E[\langle  \nabla F(\bbx^t), \bbg^t \rangle]\ge \mu\xi^t+\frac{1}{2}G_t-\frac{L^2}{2}V_t.
		\end{align*}
		Next, we bound the second term on the right side of \eqref{eq2} by conditioning on $\F^t$ as follows:
		\begin{align*}
		\E[\Vert \bbg^t \Vert^2|\F^t]\! &= \E[\Vert \overline{\nabla F}(\X^t) + \bar \e^t \Vert^2 | \F_t] \cr  
		&= \Vert \overline{\nabla F}(\X^t) \Vert^2 + \E[\Vert \bar \e^t \Vert^2| \F^t] \\&
		\!\leq\!  \frac{1}{n}\!\sum_{i=1}^n \!\Vert \nabla F(\bx_i^t) \Vert^2 \!+\! \frac{1}{n^2}\sum_{i=1}^n(\sigma^2 \!+\! c\Vert F(\bx_i^t) \Vert^2). \displaybreak
		\end{align*}
		Taking expectation from the above expression, we have
		\begin{align} \label{eq4}
		\E[\Vert \bbg^t \Vert^2]\le (1+\frac{c}{n})G_t+\frac{\sigma^2}{n}. 
		\end{align}
		Substituting \eqref{eq3}, \eqref{eq4} into \eqref{eq2} completes the proof.
	\end{proof}

	Next, we proceed to bound $V_t$. But before that, we first state and prove the following useful lemma.
	\begin{lemma}\label{lem3}
		Let $\rho$ be the second largest eigenvalue of the doubly stochastic matrix $\W$. Then, for any matrix $\Y\in \R^{d\times n}, \Y\1_n = 0$, we have $\|\Y \W\|_F^2\le\rho^2\|\Y\|_F^2$.
	\end{lemma}
	\begin{proof}
		Let $\by_1^\top,\cdots,\by_d^\top$ be the row vectors of $\Y$. By the assumption, we have $\by_i^\top\1_n=0,\forall i\in [d]$. Therefore,
		\begin{align*}
		\|\Y \W\|_F^2 &= \sum_{i=1}^d\|\by_i^\top W\|^2\\
		&\le \sum_{i=1}^d  \Big(\max_{\substack{\by^\top\1_n=0\\ \|\by\|^2=1}}\|W\by\|^2\Big)\cdot\|\by_i\|^2\\
		&= \sum_{i=1}^d \rho^2 \|\by_i\|^2 =\rho^2\|\Y\|_F^2.
		\end{align*}
	\end{proof}
	
	\begin{lemma}\label{lem2}
		Let Assumptions \ref{as: F} and \ref{as: g} hold. Then, 
		\begin{align}\nonumber
		V_{t+1}&\!\le\!\rho_t^2\Big(V_t\left(1\! -\! 2\eta_t \mu \!+\! \eta_t^2L^2 \right)\!+\!\frac{n\!-\!1}{n}\eta_t^2\sigma^2\!+\!\frac{n\!-\!1}{n}\eta_t^2 cG_t\Big).
		\end{align}
	\end{lemma}
	\begin{proof}
		Let us $Q = \frac{1}{n}\1_n\1_n^\top$. Then, we have
		\begin{align}\label{eq5}
		nV_{t+1} &= \E[\|\X^{t+1}(\bI -Q)\|_F^2]\nonumber\\
		&=\E[\|(\X^{t} - \eta_t\G(\X^t))\W_t(\bI-Q)\|_F^2]\nonumber\\
		&=\E[\|(\X^{t} - \eta_t\G(\X^t))(\bI-Q)\W_t\|_F^2]\nonumber\\
		&\le \rho_t^2\E[\|(\X^t-\eta_t\G(\X^t))(\bI-Q)\|_F^2],
		\end{align}
		where the last inequality holds by Lemma \ref{lem3}. Define $$\bx_i^{t+1/2}:=\bx_i^t-\eta_t\bg_i^t,\qquad \bbx^{t+1/2}:=\frac{1}{n}\sum_{i=1}^n\bx_i^{t+1/2}.$$
		Then, we can write
		\begin{align}\label{eq: consensus var}
		\E[\|(\X^t-\eta_t\G(\X^t))(\bI-Q)\|_F^2]&=\E [  \sum_{i=1}^{n} \Vert \bx_i^{t+1/2} - \bbx^{t+1/2} \Vert^2 ]\nonumber\\
		&= \sum_{i=1}^{n} \Vert \E[\bx_i^{t+1/2} - \bbx^{t+1/2}] \Vert^2+ \sum_{i=1}^{n} \E \left[\Vert \bx_i^{t+1/2} - \bbx^{t+1/2} - \E[\bx_i^{t+1/2} - \bbx^{t+1/2}] \Vert^2 \right].\cr 
		\end{align}
		Let us consider the first term on the right side of \eqref{eq: consensus var}. By taking conditional expectation, we obtain
		\begin{align} \label{eq: consensus 1}
		\sum_{i=1 }^{n} \Vert \E[\bx_i^{t+1/2} - \bbx^{t+1/2} | \ \F^{t}] \Vert^2 &= \sum_{i =1}^{n} \Vert \bx_i^t - \bbx^t - \eta_t(\nabla F(\bx_i^t) - \overline{\nabla F}(\X^t)) \Vert^2 
		\nonumber \\
		&=\sum_{i=1}^{n} \Vert \bx_i^t - \bbx^t \Vert^2 +\sum_{i=1}^{n} \eta_t^2\Vert\nabla F(\bx_i^t) - \overline{\nabla F}(\X^t) \Vert^2- 2 \eta_t \sum_{i=1}^{n}\langle \nabla F(\bx_i^t), \bx_i^t - \bbx^t \rangle.
		\end{align}
		By using $L$-smoothness of $F$, we have
		\begin{align} \label{eq: consensus L}
		\sum_{i=1}^n \Vert\nabla F(\bx_i^t) - \overline{\nabla F}(\X^t) \Vert^2 &= \frac{1}{n} \sum_{\{i,j\}} \Vert F(\bx_i^t) - F(\bx_j^t) \Vert^2\cr 
		&\leq \frac{L^2}{n} \sum_{\{i,j\}} \Vert \bx_i^t - \bx_j^t \Vert^2 = L^2 \sum_{i=1}^n \Vert \bx_i^t - \bbx^t \Vert^2.
		\end{align}
		Moreover, by $\mu$-strong convexity of $F$, we have
		\begin{align} \label{eq: consensus mu}
		\sum_{i=1}^n \langle F(\bx_i^t), \bx_i^t - \bbx^t \rangle &= \sum_{i=1}^n  \langle F(\bx_i^t), \frac{1}{n} \sum_{j=1}^n (\bx_i^t - \bx_j^t) \rangle \cr 
		&= \frac{1}{n} \sum_{\{i,j\}} \langle F(\bx_i^t) - F(\bx_j^t), \bx_i^t - \bx_j^t \rangle \cr 
		&\geq \frac{\mu}{n} \sum_{\{i,j\}} \Vert \bx_i - \bx_j \Vert^2 = \mu \sum_{i=1}^{n} \Vert \bx_i^t - \bbx^t \Vert^2,
		\end{align}
		where the last inequality follows from the relation $\langle \nabla F(\bx) - \nabla F(\by), \bx - \by \rangle \geq \mu \Vert \bx - \by \Vert^2$.
		Finally, by combining \eqref{eq: consensus 1}-\eqref{eq: consensus mu} we obtain,
		\begin{align*}
		\sum_{i=1 }^{n} \Vert \E[\bx_i^{t+1/2} - \bbx^{t+1/2} | \F^t] \Vert^2\leq \sum_{i=1}^{n} \Vert \bx_i^t - \bbx^t \Vert^2 \left( 1 - 2\eta_t \mu + \eta_t^2L^2 \right). 
		\end{align*}
		Next, let us consider the second term on the right side of \eqref{eq: consensus var}. We have,
		\begin{align*}
		&\sum_{i=1}^{n} \E \left[\left\Vert \bx_i^{t+1/2} - \bbx^{t+1/2} - \E[\bx_i^{t+1/2} - \bbx^{t+1/2}] \right\Vert^2 | \F^t \right] \nonumber\\
		&\qquad= 
		\sum_{i=1}^{n} \E \left[\left\Vert \bx_i^{t+1/2}- \E[\bx_i^{t+1/2}] - (\bbx^{t+1/2} - \E[\bbx^{t+1/2}]) \right\Vert^2 | \F^t \right] \\
		&\qquad= \eta_t^2 \sum_{i=1}^{n} \E \left[ \left\Vert \e_i^t - \bar \e^t \right\Vert^2 | \F^t \right] \\
		&\qquad=\eta_t^2 \left(\sum_{i=1}^{n} \E \left[ \left\Vert \e_i^t \right\Vert^2 | \F^t \right]- n \E \left[ \left\Vert \bar \e^t \right\Vert^2 |\F^t \right]\right) \\
		&\qquad= \eta_t^2 \sum_{i=1}^{n} \E \left[ \left\Vert \e_i^t \right\Vert^2 | \F^t \right](1-\frac{1}{n}) \\
		&\qquad\leq (n-1)\eta_t^2\sigma^2 + (1-\frac{1}{n})\eta_t^2 c\sum_{i=1}^n \Vert \nabla F(\bx_i^t) \Vert^2,
		\end{align*}
		where in the last equality we have used conditional independence of $\e_i^t$ to conclude $\E[\Vert \bar \e^t \Vert^2 | \F^t] = (1/n^2)\sum_{i=1}^n \E[\Vert \e_i^t \Vert^2 | \F^t]$. If we take expectation from the two relations above and combine them with \eqref{eq5} and \eqref{eq: consensus var}, we get
		\begin{align*}
		nV_{t+1} &\le \rho_t^2\E[\|(\X^t-\eta_t\G(\X^t))(\bI-Q)\|_F^2]\\
		&= \rho_t^2\sum_{i=1}^{n} \Vert \E[\bx_i^{t+1/2} - \bbx^{t+1/2}] \Vert^2 + \rho_t^2\sum_{i=1}^{n} \E \left[\Vert \bx_i^{t+1/2} - \bbx^{t+1/2} - \E[\bx_i^{t+1/2} - \bbx^{t+1/2}] \Vert^2 \right]\\
		&\le\rho_t^2\E[\sum_{i=1}^{n} \Vert \bx_i^t - \bbx^t \Vert^2 ]\left( 1 - 2\eta_t \mu + \eta_t^2L^2 \right)+\rho_t^2((n-1)\eta_t^2\sigma^2 + (1-\frac{1}{n})\eta_t^2 c\E[\sum_{i=1}^n \Vert \nabla F(\bx_i^t) \Vert^2])\\
		&=\rho_t^2(nV_t\left( 1 \!-\! 2\eta_t \mu\!+\! \eta_t^2L^2 \right)\!+\!(n\!-\!1)\eta_t^2\sigma^2\!+\!(n\!-\!1)\eta_t^2 cG_t).
		\end{align*}
		That completes the proof.
	\end{proof}

	
	\begin{lemma}\label{lem4}
		Let assumptions of Theorem \ref{thm1} hold. Then,
		\begin{align*}
		V_t \leq
		\frac{9(n-1)}{n}\sum_{k=0}^{t-1}\frac{cG_k+\sigma^2}{\mu^2(t+\beta)^2}\prod_{i=k}^{t-1}\rho_i^2.
		\end{align*}
	\end{lemma}
	\begin{proof}
		Define  $\Delta_k = (1 - 2\eta_k \mu + \eta_k^2 L^2)$ for $k\geq 0$ . Using Lemma \ref{lem2}, recursively, we can write
		\begin{align*}
		V_t &\leq \rho_{t-1}^2\Big(\Delta_{t-1}V_{t-1}  + \frac{\eta_{t-1}^2(n-1)}{n}(\sigma^2 + c G_{t-1})\Big) \\
		& \leq \rho_{t-1}^2\Delta_{t-1}\rho_{t-2}^2\Delta_{t-2}V_{t-2}\cr 
		&+\rho_{t-1}^2\rho_{t-2}^2\frac{\eta_{t-2}^2(n-1)}{n}(\sigma^2 + c G_{t-2})\cr 
		&+\rho_{t-1}^2 \frac{\eta_{t-1}^2(n-1)}{n}(\sigma^2 + c G_{t-1})\leq \ldots\cr 
		&\leq \prod_{k=0}^{t-1} \rho_k^2\Delta_k V_0 + 
		\frac{n-1}{n}\sum_{k=0}^{t-1} \eta_k^2(\sigma^2+c\E[G^k])\prod_{i=k+1}^{t-1} \Delta_i \prod_{i=k}^{t-1} \rho_i^2\\
		&= \frac{n-1}{n}\sum_{k=0}^{t-1} \eta_k^2(\sigma^2+c\E[G^k])\prod_{i=k+1}^{t-1} \Delta_i \prod_{i=k}^{t-1} \rho_i^2,
		\end{align*}
		where in the last equality we have used $V_0= 0$.
		By the choice of stepsize and $\beta\geq 2\kappa^2$,  we have 
		\begin{align*}
		\Delta_k &= 1-\frac{4}{(k+\beta)} + \frac{4L^2}{\mu^2 (k+\beta)^2} \cr 
		&\leq 1 - \frac{4}{k+\beta} + \frac{4\kappa^2}{(k+\beta)\beta}\cr &\leq 1 - \frac{4}{k+\beta} + \frac{2}{(k+\beta)} =  1 - \frac{2}{k+\beta}.
		\end{align*}
		Therefore, we have, 
		\begin{align*}
		V_t &\leq \frac{n-1}{n} \sum_{k = 0}^{t-1} \frac{4(\sigma^2+c\E[G^k])}{\mu^2 (k+\beta)^2 } \frac{(k+\beta+1)^2}{(t + \beta)^2}\prod_{i=k}^{t-1} \rho_i^2 \cr 
		&\leq \frac{n-1}{n} \sum_{k = 0}^{t-1} \frac{9(\sigma^2+c\E[G^k])}{\mu^2 (t+\beta)^2}\prod_{i=k}^{t-1} \rho_i^2,
		\end{align*}
		where in the first inequality we have used the valid inequality $\prod_{i=a}^b \left( 1 - \frac{2}{i} \right)\leq \left( \frac{a}{b+1} \right)^{2}$, and in the second inequality we have used $(k+\beta + 1)/(k+\beta)\leq (\beta+1)/\beta \leq 3/2$ since $\beta \geq 2\kappa^2 \geq 2$.
	\end{proof}
	
	\begin{lemma}\label{lem5}
		Let $b \geq a > 2$ be integers. Define $\Phi(a,b) = \prod_{i=a}^b \left( 1 - \frac{2}{i} \right)$. We then have $
		\Phi(a,b) \leq \left( \frac{a}{b+1} \right)^{2}.$
	\end{lemma}
	\begin{proof}[Proof of Lemma \ref{lem5}]
		Indeed,
		\begin{align*}
		\ln (\Phi(a,b))  =  \sum_{i=a}^b \ln \left( 1 - \frac{2}{i} \right)   
		\leq  \sum_{i=a}^b - \frac{2}{i} 
		\leq  -  2\left[ \ln (b+1) - \ln (a) \right].
		\end{align*}
		where we used the inequality $\ln (1-x) \leq -x$ as well as the standard technique of viewing $\sum_{i=a}^b 1/i$ as a Riemann sum for $\int_{a}^{b+1} 1/x ~dx$ and observing that the Riemann sum overstates the integral. Exponentiating both sides now implies the lemma.
	\end{proof}

	\newpage
	
	\section{Additional Numerical Experiments}
	In this section we provide additional experimental results for the regularized logistic regression.
	\subsection{Logistic Regression on a9a Data Set}
	We used the a9a data set from LIBSVM~\cite{chang2011libsvm} and consider logistic regression problem with $l_2$ regularization of order $\frac{1}{n}$. The objective function $F$ to be minimized is
	\begin{align}\label{eq: f-logistic}
	F(\bx) = \frac{1}{N} \sum_{j=1}^N \left( \ln(1+\exp(\bx^\top \mathbf A_j)) - 1_{(b_j = 1)} \bx^\top \mathbf A_j \right) 
	+ \frac{\lambda}{2} \Vert \bx \Vert_2^2,
	\end{align}
	where $\lambda$ is the regularization parameter, $\mathbf A_j \in \R^d$ and $b_j \in \{ 0,1 \}$, $j=1,\ldots,N$ are features (data points) and their corresponding class labels, respectively. The a9a data set consists of $N = 32561$ data points for training with $d = 123$ features. Throughout the experiments we set $\lambda=0.05$.
	
	We performed two sets of experiments for the purpose of comparing different communication strategies in Decentralized Local SGD and evaluating the impact of the number of workers and the communication network structure in Decentralized Local SGD, respectively, following the same schemes as in Section 4.
	
	Specifically, in the first set of experiments we set the number of workers to be $n=20$ and generate connected random communication graphs using an Erd$\ddot{\text{o}}$s-R$\grave{\text{e}}$nyi graph with the probability of connectivity $p = 0.3$ and the  local-degree weights to generate the mixing matrix.
	We use Decentralized Local SGD (Algorithm~\ref{alg1}) to minimize $F(\bx)$ using different communication strategies. We select $T=1000$ iterations, and the step-size sequence $\eta_t = 2/\mu(t+\beta)$ with $\beta=1,\mu=\lambda=0.05$. We start each simulation from the initial point of $\bx^0 = \mathbf 0_d$ and repeat each simulation $20$ times. The average of the results are reported in Figures \ref{fig: res4}(a) and \ref{fig: res4}(b).
	
	\begin{figure}[t!]
		\centering
		\begin{subfigure}[b]{0.49\linewidth}
			\includegraphics[width=\linewidth]{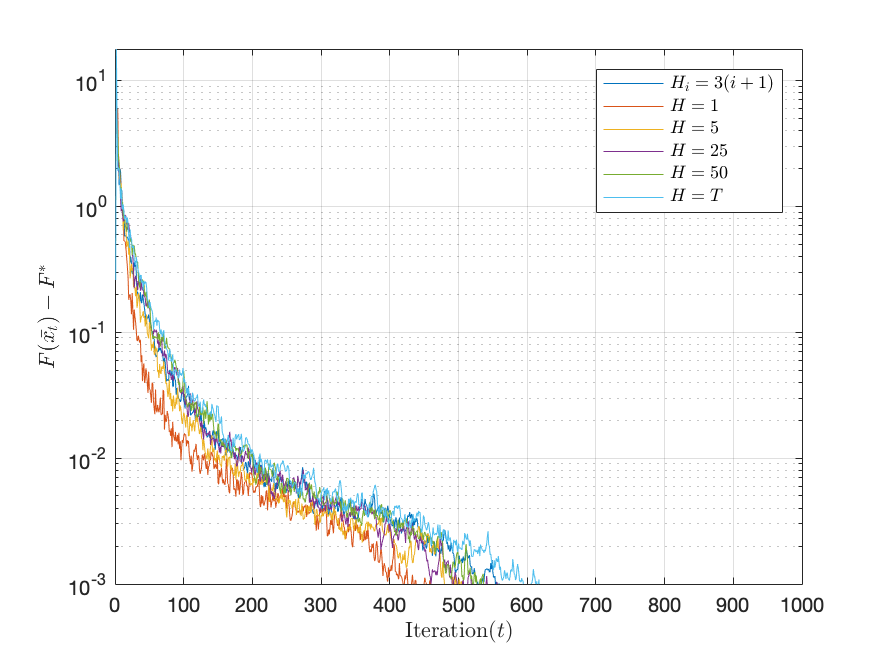}
			\caption{Error over iterations.}
		\end{subfigure}
		\begin{subfigure}[b]{0.49\linewidth}
			\includegraphics[width=\linewidth]{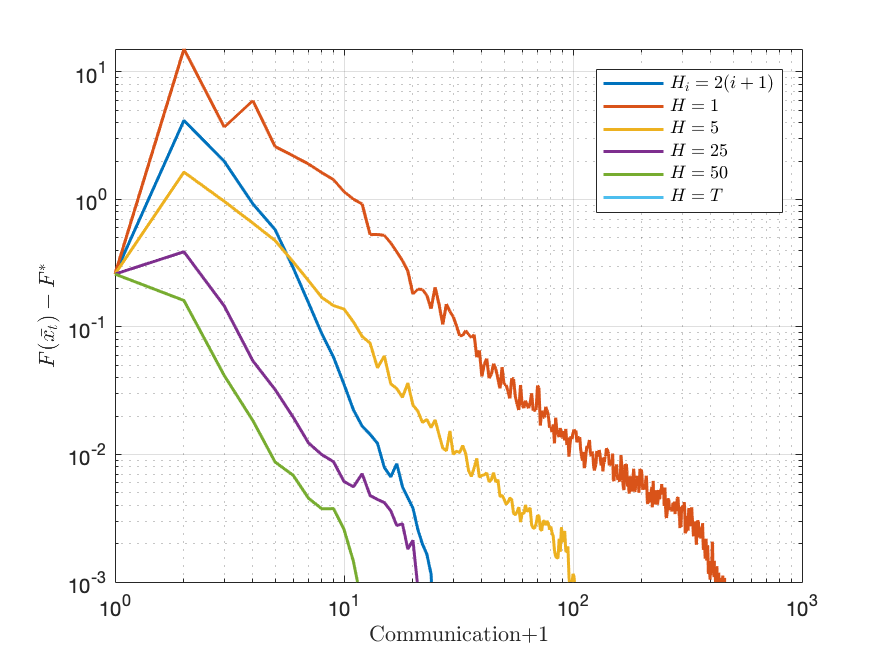}
			\caption{Error over communications.}
		\end{subfigure}
		\caption{Local SGD with different communication strategies for $l_2$ regularized logistic regression with $F(\bx)$ defined in \eqref{eq: f-logistic}, $\mu=\lambda=0.05, \beta=1$. Figures (a) and (b) show the error of different communication methods over iteration and communication round, respectively, with a fixed network size of $n=20$.}
		\label{fig: res4}
	\end{figure}
	
	Figure \ref{fig: res4}(a) shows that all the communication strategies share similar behavior considering error over iterations. This may be due to the fact that for the noise variance the strong-growth condition (Assumption \ref{as: g}) is not satisfied with a significant coefficient $c\ge 0$. Figure \ref{fig: res4}(b) shows that when considering error over communications, our proposed strategy remains one of the most competitive.
	
	\begin{figure}[t!]
		\centering
		\begin{subfigure}[b]{0.49\linewidth}
			\includegraphics[width=\linewidth]{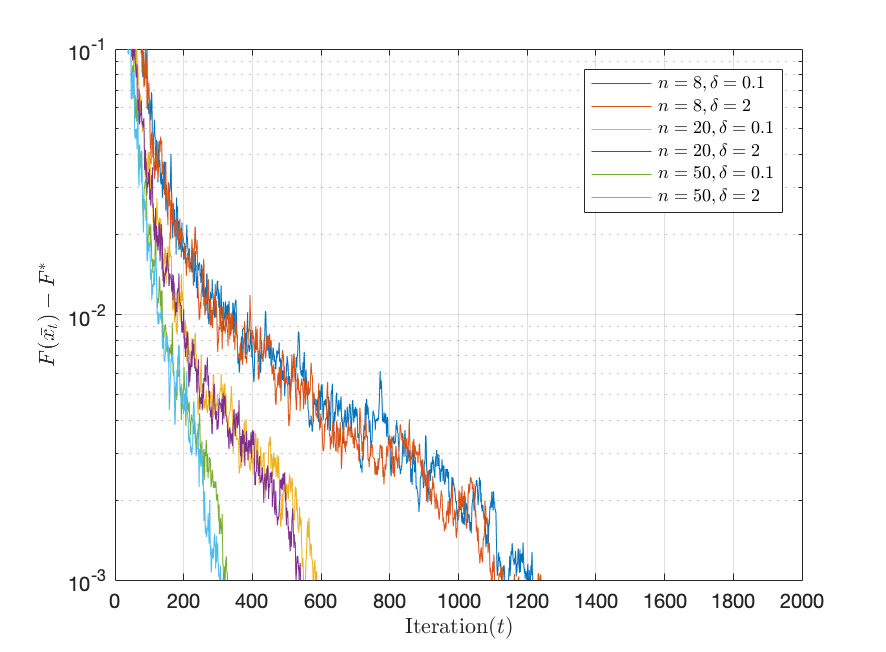}
			\caption{Error over iterations.}
		\end{subfigure}
		\begin{subfigure}[b]{0.49\linewidth}
			\includegraphics[width=\linewidth]{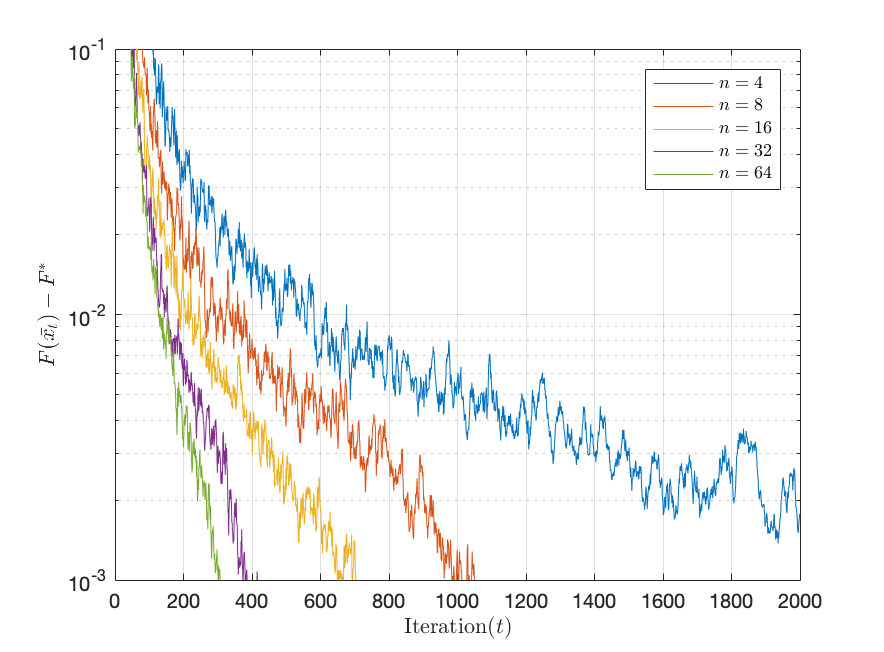}
			\caption{Error over communications.}
		\end{subfigure}
		\caption{The convergence of Local SGD for $l_2$ regularized logistic regression with the communication method proposed in this paper ($2n$ communication rounds) for various Erd$\ddot{\text{o}}$s-R$\grave{\text{e}}$nyi graphs with different number of workers and the probability of connectivity (a) and various path graphs with different number of workers (b).}
		\label{fig: res5}
	\end{figure}
	
	To evaluate the impact of the number of workers and the communication network structure in Decentralized Local SGD, we generate two sets of communication graphs: various Erd$\ddot{\text{o}}$s-R$\grave{\text{e}}$nyi graphs with different number of workers and the probability of connectivity $p = \bar{d}/n$, where $\bar{d}$ is the average degree of nodes and $\bar{d} = (1+\delta)\ln n$ with $\delta=0.1$, indicating a sparse Erd$\ddot{\text{o}}$s-R$\grave{\text{e}}$nyi graph, or $\delta=2$, indicating a dense Erd$\ddot{\text{o}}$s-R$\grave{\text{e}}$nyi graph; various path graphs with different number of workers. We use the  local-degree weights to generate the mixing matrix.
	
	We use Decentralized Local SGD (Algorithm~\ref{alg1}) to minimize $F(\bx)$ and set the communication strategy to be varying intervals with the number of communication rounds $R=2n$. We select $T=2000$ iterations, and the step-size sequence $\eta_t = 2/\mu(t+\beta)$ with $\beta=1,\mu=\lambda=0.05$. We start each simulation from the initial point of $\bx^0 = \mathbf 0_d$ and repeat each simulation $20$ times. The average of the results are reported in Figures \ref{fig: res5}(a) and \ref{fig: res5}(b). 
	
	Figures \ref{fig: res5}(a)(b) show similar patterns as Figures 2(a)(b), and further verify that while graph connectivity may not be the major impact factor of the convergence speed of Decentralized Local SGD, linear-speed up in the number of workers can be expected with only $R = 2n$ communication rounds.
	
\end{document}